\documentclass{article}

\usepackage{PRIMEarxiv}
\usepackage{authblk}

\usepackage[utf8]{inputenc} % allow utf-8 input
\usepackage[T1]{fontenc}    % use 8-bit T1 fonts
\usepackage{hyperref}       % hyperlinks
\usepackage{url}            % simple URL typesetting
\usepackage{booktabs}       % professional-quality tables
\usepackage{amsfonts}       % blackboard math symbols
\usepackage{nicefrac}       % compact symbols for 1/2, etc.
\usepackage{microtype}      % microtypography
\usepackage{fancyhdr}       % header
\usepackage{graphicx}       % graphics
\graphicspath{{media/}}     % organize your images and other figures under media/ folder

\usepackage{amsmath, amssymb, natbib, graphicx, url, algorithm2e, amsthm}
\usepackage{bbm}
\usepackage{mathtools}
\usepackage{color}
\usepackage{thm-restate}
\usepackage{times}
\usepackage{fluff} 

\allowdisplaybreaks

\newcommand{\alg}{\textsc{Alg}}
\newcommand{\opt}{\textsc{Opt}}
\newcommand{\reg}{\textsc{Regret}}

\title{A New Benchmark for Online Learning with Budget-Balancing Constraints}
  
\author[1,2]{Mark Braverman$^{*}$}
\author[1]{Jingyi Liu$^{*}$}
\author[2]{Jieming Mao}
\author[2]{Jon Schneider}
\author[1]{Eric Xue$^{*}$}

\affil[1]{Department of Computer Science, Princeton University}
\affil[2]{Google Research}

\begin{document}
\maketitle

\renewcommand{\thefootnote}{\fnsymbol{footnote}}

\footnote{Research is supported in part by the NSF Alan T. Waterman Award, Grant No. 1933331}
\renewcommand*{\thefootnote}{\arabic{footnote}}

\begin{abstract}%
The adversarial Bandit with Knapsack problem is a multi-armed bandits problem with budget constraints and adversarial rewards and costs. In each round, a learner selects an action to take and observes the reward and cost of the selected action. The goal is to maximize the sum of rewards while satisfying the budget constraint. The classical benchmark to compare against is the best fixed distribution over actions that satisfies the budget constraint in expectation. Unlike its stochastic counterpart, where rewards and costs are drawn from some fixed distribution \citep{Badanidiyuru18}, the adversarial BwK problem does not admit a no-regret algorithm for every problem instance due to the ``spend-or-save'' dilemma \citep{Immorlica22}. 

A key problem left open by existing works is whether there exists a weaker but still meaningful benchmark to compare against such that no-regret learning is still possible. In this work, we present a new benchmark to compare against, motivated both by real-world applications such as autobidding in repeated auctions and by its underlying mathematical structure. The benchmark is based on the Earth Mover's Distance (EMD), and we show that sublinear regret is attainable against any strategy whose spending pattern is within EMD $o(T^2)$ of any sub-pacing spending pattern. 

As a special case, we obtain results against the ``pacing over windows'' benchmark, where we partition time into disjoint windows of size $w$ and allow the benchmark strategies to choose a different distribution over actions for each window while satisfying a pacing budget constraint. Against this benchmark, our algorithm obtains a regret bound of $\tilde{O}(T/\sqrt{w}+\sqrt{wT})$. We also show a matching lower bound, proving the optimality of our algorithm in this important special case. In addition, we provide further evidence of the necessity of the EMD condition for obtaining a sublinear regret.
\end{abstract}

\keywords{Bandits with Knapsack \and Online Learning with Time-Varying Constraints}

%--------Main Paper Starts Here--------
\section{Introduction}

% Motivation, Setting, Model, Our Benchmark and connection to other benchmarks, Question to Investigate, Goal

The multi-armed bandits (MAB) problem is the predominant model in which the tension between exploration and exploitation has been studied for decades \citep{slivkins19introduction}.
In this model, a learner chooses an action, frequently referred to as an arm in the literature, and learns its reward at each time step.
Her goal is to maximize her total reward over finitely many time steps.
The learner confronts the tension between exploration and exploitation every time she chooses an action: should she choose an action that she already believes to yield high rewards or should she choose a potentially suboptimal action in hopes of discovering an even better alternative?

We study a generalization of the vanilla MAB setting known as the Bandits with Knapsacks (BwK) problem.
Introduced by \cite{Badanidiyuru18}, this problem differs from its standard counterpart in that each action now also consumes some amount of resources, which the learner possesses in limited supply.
Once the learner runs out of one of the resources, she must stop choosing actions.
The goal of the learner is now to maximize her total reward subject to her budget, or knapsack, constraints, hence the name Bandits with Knapsacks (BwK).
In this paper we focus on the single resource setting.
 
The BwK with a single resource setting captures a wide range of real world problems, including, but not limited to, dynamic pricing and repeated auctions \cite{Badanidiyuru18, castiglioni2024online, aggarwal2024auto}.

Prior literature has studied the BwK problem in both stochastic and adversarial settings.
In the stochastic BwK setting, the reward and cost of each action comes from some known distribution.
In contrast, in the adversarial variant, an adversary with access to the learner's algorithm (but not to her randomness) chooses the rewards and costs.
\cite{Badanidiyuru18} give an efficient algorithm for the stochastic setting that competes against the best fixed distribution over actions and show that no algorithm can perform better.
That is, let $\opt$ denote the expected total reward of the best fixed distribution over actions that spends at most the learner's budget over $T$ time steps, and let $\alg$ denote the expected total reward of aforementioned algorithm.
Then, $\opt - \alg \ll T$. 
Unfortunately, \cite{Immorlica22} prove that no algorithm can replicate this guarantee in the adversarial setting.

The lower bound example in 
\cite{Immorlica22} involves only a single resource  and captures a now well-known dilemma in the BwK literature known as the ``spend or save'' dilemma.
In the example, the learner faces only two actions over $T$ time steps subject to a budget of $T/2$ of the resource.
Action 1 always yields no reward at no cost, while action 2 always costs 1 but yields different rewards in the first and second halves of the time horizon.
In the first half, action 2 always yields reward 1/2.
In the second half, action 2 yields either no reward always or reward 1 always.
In the instance with better rewards in the second half, $\opt = 3T/8$ and is realized by the strategy that chooses each action with probability 1/2. %\markcomment{I'm confused: why not play only in the second half with probability $1$ and get reward $T/2$?}\jingyi{Re: we are competing against fixed distributions.}
In the instance with worse rewards, $\opt = T/4$ and is realized by the strategy that chooses action 2 with probability 1 (until the budget runs out, after which the game ends).
Since the two instances coincide for the first half of the time horizon, the learner must decide how much of her budget to spend on the first half without knowing what instance she is in.
It is not hard to see that as a result, $\opt - \alg \geq \Omega(T)$ on at least one of the instances for any algorithm $\alg$. 

This impossibility motivated \cite{Immorlica22} to study the competitive ratio $\opt / \alg$ instead of the regret $\opt - \alg$.
They pose finding a weaker but perhaps more reasonable benchmark to compete with as an open problem.
We address this open problem by proposing a new benchmark for BwK problems with one resource and showing that achieving $o(T)$ regret against this benchmark is possible.

\subsection{Benchmark and Results}

We say a sequence of non-negative real numbers is ``sub-pacing'' if at each time step, the corresponding value in the sequence is at most a $1/T$ fraction of the budget.
We call such sequences ``sub-pacing'' because a strategy, or a sequence of distributions over actions, that ``paces'' its usage of the resource intuitively should spend exactly a $1/T$ fraction of the budget at each time step.
Given a set of strategies of interest, we propose as a new benchmark the expected total reward of the best strategy in this set whose spending pattern, that is, the sequence formed by the strategy's expenditure at each time step, is ``close'' to the set of sub-pacing sequences that cumulatively spend the same amount of resource.

We formalize what it means to be ``close'' using the Earth Mover's Distance (EMD) and say that two spending patterns are ``close'' to each other if the EMD between them is at most some permitted distance.
So that our benchmark remains as general as possible, we leave the set of strategies of interest and the permitted distance from the set of sub-pacing spending patterns as parameters that can specified depending on the setting.
For example, our benchmark captures the expected total reward of the best fixed distribution that spends at most a $1/T$ fraction of the budget at each time step by instantiating the set of strategies to be the set of fixed distributions and the permitted distance to be 0.
Unsurprisingly, our regret guarantee against our benchmark depends on both of these parameters.

Intuitively, the EMD between two spending patterns captures how much and how far resource has to be moved in order to build one spending pattern using resource spent by the other.
By choosing the permitted distance from the set of sub-pacing spending patterns to be sufficiently small, our benchmark excludes strategies like those behind the ``spend or save'' dilemma from contributing to the benchmark.
More concretely, in the lower bound example of \cite{Immorlica22}, one of the benchmark strategies spends one unit of resource per time step in the first half.
The per-time-step budget is $1/2$, so the EMD between this spending pattern and any sub-pacing spending pattern would be $\Omega(T^2)$.
Choosing the permitted distance to be $o(T^2)$ would thus exclude their problematic benchmark strategy.

Our main result is an algorithm, which we call $\texttt{LagrangianEMD}$, that obtains the following regret guarantee.

\smallskip
\noindent{\bf Theorem \ref{theorem:EMD-main-result} (restated, informal).}
%\label{theorem:EMD-regret-guarantee-intro}
Let $F$ denote the set of strategies of interest, and let $D$ denote the permitted distance from the set of sub-pacing spending patterns.
Let $\opt_{D, F}$ denote our benchmark, which, in the language of $F$ and $D$, is the reward of the best strategy in $F$ that is within EMD $D$ of the set of sub-pacing spending patterns.
When the reward-to-cost ratio is bounded,
\[
    \opt_{D,F} - \texttt{LagrangianEMD} \leq O(\sqrt{T \log \abs{F}} + \sqrt{D})
\]

As its name suggests, $\texttt{LagrangianEMD}$ falls into the family of primal-dual-based approaches, which have been successful in tackling both stochastic and adversarial BwK problems.
Like other primal-dual-based approaches, \texttt{LagrangianEMD} sets up a zero-sum sequential game between a primal player and a dual player, which are represented by two no-regret learning algorithms.
At each time step, the primal player chooses a distribution over actions, while the dual player chooses a Lagrangian multiplier $\lambda_t$ corresponding to a per-round budget constraint.
Payoffs are given by Lagrangified rewards.
More specifically, the payoff to the primal player is $r_t + \lambda_t (B/T - c_t)$ and the payoff to the dual player is $-\lambda_t (B/T - c_t)$.
Here, $r_t$ denotes the reward, $c_t$ denotes the cost, and $B$ denotes the budget.

% The key observation that enables our main result is that two spending patterns that are close in EMD offer comparable payoff to the dual player.
% In particular, we show that the difference in payoff to the dual player is bounded by the stability of her strategy times the EMD between the two spending patterns.
% Thus, properly tuning the step size of the dual player's learning algorithm allows us to control the payoff of the dual player against the optimal primal strategy.
% The no-regret guarantee of the primal player's learning algorithm then allows us to relate her payoff to the payoff of her best response.

By design, the dual player's payoff is 0 if the primal player chooses a ``perfectly" pacing strategy that spends exactly $B/T$ in each round.
We wish to ensure that this  (approximately) remains the case even when the primal player's spending is ``approximately uniform" over time. 
We note that in order to achieve this, the only assumption we make is on the rate of change of the dual player's strategy $|\lambda_t-\lambda_{t+1}|$.

A dual player whose rate of change is limited leads to a strategy $\lambda_t$ that is a {\em Lipschitz continuous} function in $t$. The payoff of the dual player is thus an inner product of the spending trajectory $c_t$ and $\lambda_t$. Since the dual norm to Lipschitz continuity is precisely the Earth Mover's Distance \citep{villani2003topics} (also known as the Wasserstein metric $W_1$ in the literature), the set of $c_t$'s for which $$\textstyle
\sum_t c_t\cdot  \lambda_t \approx \sum_t (B/T)\cdot \lambda_t$$
{\em for all permitted  sequences $(\lambda_t)_{t=1}^{T}$} is precisely the set of $c_t$'s that are close to uniform spending in EMD distance. 

To demonstrate the potential of \texttt{LagrangianEMD} as a general solution to BwK problems with one resource, we show that it yields non-trivial regret guarantees against a previously studied benchmark\footnote{Although we obtain non-trivial regret guarantees against this benchmark, our results should not be directly compared because our setting is slightly different.}.
We also go beyond these benchmarks and show that \texttt{LagrangianEMD} achieves the optimal regret guarantee (up to logarithmic terms) against the following natural benchmark.

Divide the entire time horizon of length $T$ into $T/w$ disjoint intervals of length $w$.
The benchmark is the expected total reward of the strategy that in each interval, plays according to the best fixed distribution over actions that spends at most a $w/T$ fraction of the budget in the interval.
Let $\opt$ denote this benchmark.
We demonstrate the following regret guarantee for \texttt{LagrangianEMD}.

\smallskip
\noindent{\bf Theorem \ref{thm:window-EMD-algo} (restated, informal).}
\texttt{LagrangianEMD} parameterized by the permitted distance $O(wT)$ and the set of strategies that in each interval, plays according to a fixed distribution achieves regret
\[
    \opt - \texttt{LagrangianEMD} \leq \Tilde{O}(T / \sqrt{w} + \sqrt{wT})
\]

We show that this regret guarantee is nearly tight in the following theorem, demonstrating that there exists instances for which the dependence in Theorem~\ref{theorem:EMD-main-result} on the permitted distance from the set of sub-pacing spending patterns cannot be improved.

\smallskip
\noindent{\bf Theorem \ref{theorem:lower-bound} (restated, informal).}
There exists an instance such that for any algorithm,
\[
    \opt - \alg \geq \Omega(T / \sqrt{w} + \sqrt{wT})
\]

Like the lower bound example of \cite{Immorlica22}, our construction only involves two arms and one resource and is quite natural.
Rewards are generated via an unbiased  random walk, while costs stay constant.
Consequently, at any point in time, the expected future reward-to-cost ratio coincides with the reward-to-cost ratio of the current time step.

\section{Related Work}
\subsection{Bandits with Knapsacks}
\cite{Badanidiyuru18} introduced the Bandits with Knapsacks (BwK) problems in the stochastic setting, where rewards and costs are jointly drawn from some fixed distribution. At each time $t$, the player pulls an arm and observes the reward and cost for that arm. The game stops when the player's budget $B$ has been depleted (called the hard-stopping assumption). \cite{Badanidiyuru18} provided an optimal algorithm and a matching lower bound in their original paper. Other algorithms have been developed since then that achieved the same optimal guarantee \citep{agrawal2014bandits, agrawal2019bandits, castiglioni2022unifying}. \cite{Immorlica22} extended the BwK problem to the adversarial setting, where both rewards and costs are adversarially chosen. In this case, no regret against the best fixed distribution over arms that satisfies the budget constraint is no longer attainable. They designed a Lagrangian-based algorithm that achieved a competitive ratio of $\Omega(K\log T)$, where $K$ is the number of arms, and showed a matching lower bound. This algorithm is also optimal for the stochastic setting when the budget $B=\Theta(T)$, but suboptimal when $B$ or the optimal sum of rewards is $o(T)$. 

Follow-up work by \cite{kesselheim2020online} improved the dependence on $K$ to $\log K$ and \cite{castiglioni2022online} achieved a constant competitive ratio when $B=\Theta(T)$. \cite{castiglioni2022unifying} achieved a constant competitive ratio (under strict feasibility assumptions) by allowing vanishing constraint violations instead of imposing hard-stopping. Under the regime $B=\Theta(T)$, \cite{castiglioni2024online} considered both budget and ROI constraints for the auto-bidding setting and relaxed the strict feasibility assumption from prior work \citep{Immorlica22, balseiro2020dual}. There are also a few works that consider environments that fall between stochastic and adversarial, for example the switching environment in \cite{slivkins2023contextual} and the environments with a soft non-stationarity measure in \cite{liu2022non}. Most works in the adversarial setting compete against the best fixed distribution over arms that satisfies the long-term constraint, where the best hope is a constant competitive ratio, while our work circumvents this impossibility by competing against strategies that spend their budget in a pacing manner.

\subsection{Online Convex Optimization with Constraints}
\label{sec: rel-work-OLC}
Another closely related line of work looks at online convex optimization problems with convex constraints. At each time $t$, the player chooses a strategy $x_t$ from some feasible convex set $\mathcal{X}$ and observes a convex cost function $f_t(\cdot)$ and constraint function $g_t(\cdot)\leq 0$. Different from the classical Bandits with Knapsack setting, the online learning algorithm is not allowed to stop when the cumulative violation of constraints (CCV) is above 0. Thus the goal here becomes both minimizing the sum of costs $\sum_t f_t(x_t)$ against some benchmark strategy and minimizing the cumulative violation $\sum_t g_t(x_t)$. The usual benchmark to compare against is the best fixed strategy $x^*$ s.t. $g_t(x^*)\leq 0, \forall t\in [T]$. A number of papers looked at constraint functions that are either time-invariant or i.i.d. drawn from some fixed distribution \citep{mahdavi2013stochastic, jenatton2016adaptive, yuan2018online, yu2017online, yu2020low, wei2020online, yi2021regret, guo2022online}. \cite{guo2022online} obtained the best known $O(\sqrt{T})$ regret and $O(1)$ CCV bounds without assuming the existence of a strictly feasible strategy for the constraints. 

More relevant to us are works that look at adversarial constraint functions. This setting was first introduced by \cite{mannor2009online} who showed that it is impossible to compete against the best fixed strategy $x^*$ that satisfies the cumulative constraint $\sum_t g_t(x_t)\leq 0$ while ensuring a vanishing cumulative violation even for linear cost and constraint functions. Later works impose more restrictions on the benchmark, by requiring the benchmark strategy $x^*$ to be feasible at all times, i.e., $g_t(x^*)\leq 0, \forall t\in [T]$ \citep{neely2017online,sun2017safety,chen2018bandit, sinha2024optimal}.
\cite{sinha2024optimal} achieved the optimal bounds (up to log factors) of $O(\sqrt{T})$ regret and $O(\sqrt{T}\log T)$ CCV (which they defined as $\sum_{t=1}^{T}(g_t(x_t))^{+}$). However, this benchmark class is very restrictive and excludes strategies that satisfy the cumulative constraints in a longer window. \cite{liakopoulos2019cautious} considered a slightly more general benchmark where any benchmark strategy needs to satisfy the cumulative constraints in any window of size $K$. In this setting, they achieved a parameterized regret bound of $O(KT/V+\sqrt{T})$ and CCV bound of $O(\sqrt{VT})$ for any $V\in [K,T)$ (choosing $V=K^{2/3}\cdot T^{1/3}$ leads to a bound of $O(K^{1/3}\cdot  T^{2/3})$ for both). The results are incomparable to ours for several reasons: (1) their ``pacing over sliding windows" benchmark is a subset of our ``pacing over windows" benchmark in Section~\ref{sec:benchmark-interval}; (2) we allow different strategies to be chosen over disjoint windows; (3) unlike the BwK setting, they are not allowed to stop once the budget constraint is violated --- this leads to a worse bound than the $O(K^{1/2} \cdot T^{1/2})$ we would obtain in this setting.

\section{Bandits with Knapsacks and One Resource}

\subsection{Problem Formulation and Assumptions}

In the Bandits with Knapsacks (BwK) setting with a single resource, a learner has $B$ units of a single resource to spend on a (finite) set of actions $A$ over $T$ rounds.
More specifically, in each round $t$, the learner chooses an action $a_t \in A$ and incurs the reward $r_t(a_t) \in [0,1]$ and the cost $c_t(a_t) \in [0,1]$ for the chosen action. 
Since we are in the bandits setting, this reward and cost is the only information that the learner observes. 
In contrast, in the full information setting, the learner additionally observes the rewards $r_t(a)$ and costs $c_t(a)$ for the other actions $a \in A$.
The learner's objective is to maximize her sum of rewards $\sum_{t=1}^T r_t(a_t)$ subject to her budget $B$. 
We assume that the set of actions $A$ contains a null action $\perp$ that the learner can choose for free to obtain no reward. We also assume that there exists $\alpha \geq 0$ such that $r_t(a) \leq \alpha \cdot c_t(a)$ for all time steps $t$ and actions $a\in A$. For the budget, we assume without loss of generality that $B\leq T$ since otherwise, the learner would not be budget-constrained. 

In each round, the learner may choose her action randomly rather than deterministically.
We use $\Delta(A)$ to denote the set of probability distributions over the set of actions $A$.
If the learner chooses an action in round $t$ according to the distribution $x_t\in \Delta(A)$, then her expected reward is $r_t(x_t) \coloneq \sum_{a \in A} r_t(a) \cdot x_t(a)$  and her expected cost is $c_t(x_t) \coloneq \sum_{a \in A} c_t(a) \cdot x_t(a)$. 

\subsection{The EMD Benchmark and Regret Definition}

We formalize our benchmark using the notion of Earth Mover's Distance (EMD).
Given two probability distributions $D$ and $D'$ whose supports lie within the metric space $(X, d)$, the EMD between $D$ and $D'$ is defined as
\[
    \textstyle \mathrm{EMD}(D, D') \coloneq \inf_{\Pi} \EE_{(x,y) \sim \Pi}[d(x, y)]
\]
where the infimum is taken over joint distributions $\Pi$ whose marginals are $D$ and $D'$.
When the metric space is $(\RR, \abs{\cdot})$, that is, the real numbers equipped with the absolute difference metric, the EMD has a closed-form expression:
\[
    \textstyle  \mathrm{EMD}(D, D') = \int_{-\infty}^\infty {\abs{F_D(x) - F_{D'}(x)}} \, \mathrm{d} x
\]
where $F_D$ and $F_{D'}$ denote the cumulative distribution functions of $D$ and $D'$, respectively.

We use the EMD to quantify the distance between two spending patterns $(c_t)_t$ and $(d_t)_t$ such that $\sum_t c_t = \sum_t d_t$.
Let $m$ denote the value of these two summations.
Note that $(c_t)_t / m$ and $(d_t)_t / m$ constitute discrete probability distributions with support in $[T]$.
The EMD between these two distributions is 
\[
    \textstyle \mathrm{EMD}((c_t)_t / m, (d_t)_t / m) = \frac{1}{m} \sum_{t=1}^T \abs{\sum_{s=1}^t (c_s - d_s)}
\]
Because we will define our benchmark using the unnormalized EMD, we overload notation for simplicity and define
\[
    \textstyle \mathrm{EMD}((c_t)_t, (d_t)_t) \coloneq \sum_{t=1}^T \abs{\sum_{s=1}^t (c_s - d_s)}
\]
Now, given $D \geq 0$, define
\[
    \textstyle G(D) \coloneq \CrBr{(x_t)_t : \exists\, (d_t)_t : \forall\,t, d_t \in [0, \frac{B}{T}] \wedge \sum_t d_t = \sum_t c_t(x_t) \wedge \mathrm{EMD}((c_t(x_t))_t, (d_t)_t) \leq D}
\]
to be the set of strategies $(x_t)_t$ (where each $x_t \in \Delta(A)$) that are within EMD $D$ of a ``sub-pacing" spending pattern $(d_t)_t$ that has the same cumulative spending as $(x_t)_t$.
The spending pattern $(d_t)_t$ is sub-pacing if $0 \leq d_t \leq B/T$ for all time steps $t$.
Our benchmark is parameterized by $D \geq 0$ and a set of strategies $F \subseteq \Delta(A)^T$ and is given by
\[
     \textstyle \textsc{Opt}_{D, F} \coloneq \max_{x \in F \cap G(D)} \sum_{t=1}^T r_t(x_t)
\]
We define the regret of a strategy $x \coloneq (x_t)_t$ against $\opt_{D, F}$ to be
\[
    \textstyle \textsc{Regret}_{\opt_{D,F}}(x) \coloneq \opt_{D, F} - \sum_{t=1}^T r_t(x_t)
\]
At times, we will also be interested in the expected budget violation of a strategy.
We define the expected budget violation up to time step $t$ of the strategy $x$ to be
\[
    \textstyle V_t(x) \coloneq \Paren{\sum_{s=1}^t (c_s(x_s) - B/T)}^+
\]
Note that $V_T(x)$ denotes the expected total budget violation.

\subsection{Benchmark and the Spend or Save Dilemma}
\label{sec:benchmark-dilemma}

Note that our benchmark circumvents the ``spend or save'' dilemma because a strategy that spends the entire budget in the first half of the time horizon is not close in EMD to any sub-pacing spending pattern.
More concretely, recall that for one of the two instances behind the lower bound construction given by \cite{Immorlica22}, the optimal strategy spends one unit of resource per time step in the first half and nothing in the second (we denote this spending pattern by $(\Vec{1}, \Vec{0})$).
The per-time-step budget in their construction is 1/2, so any sub-pacing spending pattern $(d_t)_t$ must have $0 \leq d_t \leq 1/2$ for all time steps $t$.
The distance between the optimal strategy and any sub-pacing spending pattern is then
\begin{align*}
    \textstyle \mathrm{EMD}((\Vec{1}, \Vec{0}), (d_t)_t)
        &= \textstyle \sum_{t=1}^{T/2} \abs{\sum_{s=1}^t (1 - d_s)} + \sum_{t=T/2+1}^{T} \abs{\sum_{s=1}^{T/2} (1 - d_s) - \sum_{s=T/2+1}^{t} d_s} \\
        &\geq \textstyle \sum_{t=1}^{T/2} \abs{\sum_{s=1}^t (1 - d_s)}\\
        &= \textstyle \Omega(T^2) \tag{$d_t \leq 1/2$ for all $t$}
\end{align*}
Thus, our benchmark excludes the problematic strategies behind the lower bound of \cite{Immorlica22} when $D = o(T^2)$.

\section{Algorithm \texttt{LagrangianEMD}}
%full information setting and Bandit setting

In this section, we give an algorithm \texttt{LagrangianEMD} and bound its regret against $\opt_{D,F}$.
\texttt{LagrangianEMD} is given in Algorithm~\ref{algorithm:solution-EMD-subroutine}.

\begin{algorithm}    \caption{\texttt{LagrangianEMD}}\label{algorithm:solution-EMD-subroutine}
    \DontPrintSemicolon
    \KwIn{budget $B$, time horizon $T$, bound on dual action $\overline{\lambda}$, bound on EMD $D$, set of strategies $F \subseteq \Delta(A)^T$}
    Let $\alg_P$ be EXP4-IX (\cite{neu15-exp4-ix}; tuned to minimize regret against the best expert in $F$)\;
    Let $\alg_D$ be OGD over $[0, \overline{\lambda}]$ with $\lambda_0 = 0$ and step size  $\eta = \overline{\lambda}/ \sqrt{\max\{D, T\}}$\;
    \For{$t = 1, \dots, T$}{
        $x_t \gets \alg_P$, $\lambda_t \gets \alg_D$\;
        Draw $a_t \sim x_t$\;
        Incur reward $r_t(a_t)$ and cost $c_t(a_t)$\;
        Feed reward $r_t(a_t) + \lambda_t (B/T - c_t(a_t))$ to $\alg_P$\;
        Feed loss $\lambda_t (B/T - c_t(a_t))$ to $\alg_D$\; 
    }
\end{algorithm}

\begin{theorem}\label{theorem:EMD-main-result}
Let $D \geq 0$ and $F \subseteq \Delta(A)^T$.
If there exists a scalar $\alpha \geq 0$ such that $r_t(a) \leq \alpha \cdot c_t(a)$ for all time steps $t$ and actions $a$, then the regret against $\opt_{D,F}$ of choosing actions according to $\texttt{LagrangianEMD}(D, F)$ with $\overline{\lambda} = \alpha$ until no budget remains (and then choosing the null action for the remaining time steps) is at most $\alpha \cdot O(\sqrt{T \abs{A} \log \abs{F}} \cdot \log(1/\delta) + \sqrt{D})$ with probability $1 - \delta$ for all $\delta$ simultaneously.
\end{theorem}

\noindent\textbf{Remark} 
If one uses EXP4.P \citep{pmlr-v15-beygelzimer11a} instead of EXP4-IX in \texttt{LagrangianEMD}, then one can guarantee that the regret is at most $\alpha \cdot O(\sqrt{T \abs{A} \log (\abs{F} / \delta)} + \sqrt{D})$ with probability $1 - \delta$ for a pre-specified $\delta$ (but not for all $\delta$ simultaneously).

\begin{proof}
Let $x^* \coloneq \arg\max_{x \in F \cap G(D)} \sum_{t=1}^T r_t(x_t)$ denote the optimal strategy in $F \cap G(D)$, and let $a \coloneq (a_t)_t$ denote the actions output by \texttt{LagrangianEMD}.
The regret guarantee of $\alg_P$ implies that with probability at least $1 - \delta$,
\begin{equation*}
    \textstyle \sum\limits_{t = 1}^{T}  (r_t(x^*_t) + \lambda_t(B/T - c_t(x^*_t)) - \sum\limits_{t = 1}^{T}  (r_t(a_t) + \lambda_t(B/T - c_t(a_t))) \leq O(\overline{\lambda}\sqrt{T \abs{A} \log \abs{F}} \cdot \log(1/\delta))
\end{equation*}
Now, let $(d_t)_t$ denote the sub-pacing spending pattern that cumulatively spends exactly as much as $x^*$ and is within EMD $D$ of $x^*$ (which exists because $x^* \in F \cap G(D)$).
As a sub-pacing spending pattern, $d_t \in [0,B/T]$ for all time steps $t$.
For each time step $t$, we also have that $\abs{\lambda_{t+1} - \lambda_t} \leq \eta$ since $\alg_D$ is OGD with step size $\eta$.
It follows that
\begin{align*}
    \textstyle \sum_{t=1}^T \lambda_t(c_t(x^*_t) - B/T)
        &\leq \textstyle \sum_{t=1}^T \lambda_t(c_t(x^*_t) - d_t) \tag{$\lambda_t \geq 0$, $d_t \leq B/T$} \\
        &= \textstyle \lambda_T \sum_{t=1}^T (c_t(x^*_t) - d_t) - \sum_{t=1}^{T-1} (\lambda_{t+1} - \lambda_t) \sum_{s = 1}^t (c_s(x^*_s) - d_s) \tag{summation by parts} \\
        &= \textstyle \sum_{t=1}^{T-1} (\lambda_{t+1} - \lambda_t) \sum_{s = 1}^t (d_s - c_s(x^*_s)) \tag{$\sum_{t=1}^T c_t(x^*_t) = \sum_{t=1}^T d_t$} \\
        &\leq \textstyle \eta \sum_{t=1}^{T-1} \abs{\sum_{s = 1}^t (c_s(x^*_s) - d_s)} \tag{H\"{o}lder's, $\abs{\lambda_{t+1} - \lambda_t} \leq \eta$} \\
        &= \textstyle \eta \cdot \mathrm{EMD}((c_t(x^*_t)_t, (d_t)_t) \tag{$\sum_{t=1}^T c_t(x^*_t) = \sum_{t=1}^T d_t$} \\
        &\leq \textstyle \eta D \tag{$\mathrm{EMD}((c_t(x^*_t)_t, (d_t)_t) \leq D$}
\end{align*}
Rearranging the regret guarantee of $\alg_P$ and applying the above inequality yields that with probability at least $1 - \delta$,
\begin{equation}\label{equation:algorithm-reward}
    \textstyle \sum\limits_{t = 1}^{T}  r_t(x^*_t) - \sum\limits_{t=1}^T r_t(a_t) \leq  \eta D + \sum\limits_{t = 1}^{T} \lambda_t (B/T - c_t(a_t))) + O(\overline{\lambda}\sqrt{T \abs{A} \log \abs{F}} \cdot \log(1/\delta))
\end{equation}
But our expected total reward is not $\sum_{t=1}^T r_t(a_t)$ since we stop choosing the action recommended by $\texttt{LagrangianEMD}$ and choose the null action instead after spending the entire budget (in expectation).
To bound the loss in reward, we appeal to the regret guarantee of $\alg_D$, which, recall, is OGD over $[0, \overline{\lambda}]$ with step size $\eta$.
OGD with these parameters guarantees that
\begin{equation}\label{equation:dual-regret-EMD}
    \textstyle \sum_t \lambda_t(B/T - c_t(a_t)) - \overline{\lambda} \left(B-\sum_t c_t(a_t)\right) \leq \frac{\overline{\lambda}^2}{\eta} + \eta T
\end{equation}
Now, let $S$ denote the last time step before the expenditure of \texttt{LagrangianEMD} exceeds the budget $B$.
That is, $\sum_{t=1}^{S} c_t(a_t) \leq B$ but $\sum_{t=1}^{S+1} c_t(a_t) > B$.
Note that our strategy is $(a_1, \dots, a_S, \perp, \dots, \perp)$, which we denote as $(a_{1 : S}, \perp)$.
Thus, with probability at least $1 - \delta$ and up to an additive $O(\overline{\lambda}\sqrt{T \abs{A} \log \abs{F}} \cdot \log(1/\delta))$ factor, our regret is
\begin{align*}
    \textstyle \reg_{D,F} ((a_{1:S}, \perp))
        &= \textstyle \sum_{t = 1}^{T}  r_t(x^*_t) - \sum_{t=1}^S r_t(a_t) \\
        &\leq \textstyle \eta D + \sum_{t = 1}^{T} \lambda_t \Paren{\tfrac{B}{T} - c_t(a_t)} + \sum_{t=S+1}^T r_t(a_t) \tag{Equation~\ref{equation:algorithm-reward}} \\
        &\leq \textstyle \eta D + \sum_{t = 1}^{T} \lambda_t \Paren{\tfrac{B}{T} - c_t(a_t)} + \overline{\lambda} \sum_{t=S+1}^T c_t(a_t) \tag{$r_t \leq \alpha \cdot c_t$; $\overline{\lambda} = \alpha$} \\
        &\leq \textstyle \eta D + \sum_{t = 1}^{T} \lambda_t \Paren{\tfrac{B}{T} - c_t(a_t)} - \overline{\lambda} \Paren{B - \sum_{t=1}^T c_t(a_t)} \tag{def. of $S$} \\
        &\leq \textstyle \eta D + \frac{\overline{\lambda}^2}{\eta} + \eta T \tag{Equation~\ref{equation:dual-regret-EMD}} \\
        &\leq \textstyle O(\overline{\lambda} \sqrt{\max\{D, T\}}) \tag{$\eta = \overline{\lambda} / \sqrt{\max\{D, T\}}$}
\end{align*}
\end{proof}

\begin{corollary}
Let $D \geq 0$ and $F \subseteq \Delta(A)^T$.
If there exists a scalar $\alpha \geq 0$ such that $r_t(a) \leq \alpha \cdot c_t(a)$ for all time steps $t$ and actions $a$, then the expected regret against $\opt_{D,F}$ of choosing actions according to $\texttt{LagrangianEMD}(D, F)$ with $\overline{\lambda} = \alpha$ until no budget remains (and then choosing the null action for the remaining time steps) is at most $\alpha \cdot O(\sqrt{T \abs{A} \log \abs{F}} + \sqrt{D})$.
\end{corollary}

\begin{proof}
Let $S$ denote the last time step before the expenditure of \texttt{LagrangianEMD} exceeds the budget $B$, and let $\varphi(x) \coloneq \PP[\reg_{D,F} ((a_{1:S}, \perp)) \leq x]$ denote the CDF of $\reg_{D,F} ((a_{1:S}, \perp))$.
By Theorem~\ref{theorem:EMD-main-result}, we can bound the expected regret as follows.
\begin{align*}
    \textstyle \EE[\reg_{D,F} ((a_{1:S}, \perp))]
        &= \int_0^1 \varphi^{-1}(1 - \delta) \,\mathrm{d}\delta \\
        &\leq \int_0^1  \alpha \cdot O(\sqrt{T \abs{A} \log \abs{F}} \cdot \log(1/\delta) + \sqrt{D}) \,\mathrm{d}\delta \\
        &= \alpha \cdot O(\sqrt{T \abs{A} \log \abs{F}} + \sqrt{D})
\end{align*}

\end{proof}

\section{Applications of Algorithm \texttt{LagrangianEMD}}

In this section, we show that \texttt{LagrangianEMD}
achieves non-trivial regret guarantees against the pacing benchmark of \cite{liakopoulos2019cautious} in our setting.
We emphasize that the purpose of this exercise is not to compare our regret bound against theirs (our results are incomparable because our settings are different) but to demonstrate the applicability of our algorithm.
We further demonstrate its potential as a general solution to BwK problems with one resource by proving that it achieves optimal regret guarantees for a new pacing benchmark.

\subsection{Best Fixed Distribution and Sliding Window Budget Constraint}

In this section, we show that \texttt{LagrangianEMD} obtains a non-trivial regret guarantee against the following pacing benchmark.
\begin{align*}
     \textstyle \textsc{Opt} \coloneq \max_{x \in \Delta(A)} \left\{\sum_{t = 1}^{T} r_t(x) : \forall\, k=1, \dots, T-w+1, \sum_{t = k}^{k+w-1}  c_t(x) \leq Bw/T\right\}
\end{align*}
In other words, $\opt$ is the reward obtained by the best fixed strategy that spends at most $Bw/T$ in every interval of length $w$.
We show that the set of such strategies is within EMD $O(wT)$ of sub-pacing by proving an even stronger result: the set of all strategies that spend at most $Bw/T$ in each interval $[kw+1, (k+1)w]$ where $k = 0, \dots, T/w-1$ is within EMD $O(wT)$ of sub-pacing.

\begin{lemma}\label{lemma:window-benchmark-strategies-within-EMD-wT}
Let $\mathcal{F} \subseteq \Delta(A)^T$ denote the set of all strategies that spend at most $Bw/T$ in each interval $[kw+1, (k+1)w]$ where $k = 0, \dots, T/w-1$.
Then, $\mathcal{F} \subseteq G(O(wT))$.
\end{lemma}

\begin{proof}
Let $x \in \mathcal{F}$, and for all $k = 0, \dots, T/w-1$, let $C_k \coloneq \sum_{t=kw+1}^{(k+1)w} c_t(x_t)$ denote how much $x$ spends in $[kw+1, (k+1)w]$.
Since $x \in \mathcal{F}$, we know that $C_k \leq Bw/T$.
Thus, $(d_t)_t$ where $d_t = C_{\ceil{t/w} - 1} / w$ is a sub-pacing spending pattern such that $\sum_{t=1}^T c_t(x_t) = \sum_{t=1}^T d_t$.
It remains to show that the EMD between $(c_t(x_t))_t$ and $(d_t)_t$ is at most $O(wT)$.
\begin{align*}
    \textstyle \mathrm{EMD}((c_t(x_t))_t, (d_t)_t)
        &= \textstyle \sum_{t=1}^T \abs{\sum_{s=1}^t (c_s(x_s) - d_s)} \\
        &= \textstyle \sum_{t=1}^T \abs{\sum_{s = (\ceil{t/w} - 1)w + 1}^t (c_s(x_s) - d_s)} \tag{$C_k = \sum_{t=kw+1}^{(k+1)w} d_t$} \\
        &\leq \textstyle \sum_{t=1}^T (t - (\ceil{t/w} - 1)w) \tag{$c_t \in [0,1]^A$; $d_t \in [0, 1]$} \\
        &\leq wT
\end{align*}
\end{proof}

\begin{theorem}
Let $\Delta(A)_\varepsilon$ denote an $\varepsilon$-cover (with respect to the $\ell_\infty$ norm) of $\Delta(A)$ of minimum size, and let $F \subseteq \Delta(A)_{\varepsilon / \abs{A}}^T$ denote the set of fixed distributions over actions.
If there exists $\alpha \geq 0$ such that $r_t(a) \leq \alpha \cdot c_t(a)$ for all time steps $t$ and actions $a$, then the regret against $\opt$ of choosing actions according to $\texttt{LagrangianEMD}(F)$ with $D = O(wT)$ and $\overline{\lambda} = \alpha$ until no budget remains (and then choosing the null action for the remaining time steps) is at most $\alpha \cdot O(\abs{A} \sqrt{T\log(\abs{A}/\varepsilon)} \cdot \log(1/\delta) + \sqrt{wT} + \varepsilon T)$ with probability at least $1 - \delta$ for all $\delta$ simultaneously.
If $\varepsilon = \sqrt{w/T}$, then the regret is at most $\alpha \cdot O(\abs{A} \sqrt{T \log(T \cdot \abs{A} / w)} \cdot \log(1/\delta) + \sqrt{wT})$ with probability at least $1 - \delta$ for all $\delta$ simultaneously.
\end{theorem}

\begin{proof}
Let $F' \subseteq \Delta(A)^T$ denote the set of fixed distributions that spend at most $Bw/T$ in every interval of length $w$.
Let $x^* = \arg\max_{x \in F'} \sum_t r_t(x_t)$ denote the optimal fixed distribution ($x^*_t$ is constant over $t\in [T]$), and let $x \in F$ denote a strategy such that $x^*_t(a) - x_t(a) \in [0, 2\varepsilon/\abs{A}]$ for all time steps $t$ and actions $a$ (put the remaining mass on the null arm).
Note that $\textstyle \sum_{t=1}^T r_t(x_t) \geq \sum_{t=1}^T r_t(x^*_t) - 2 \varepsilon T$, so it suffices to bound our regret against $x$.

Since $c_t(x_t) \leq c_t(x_t^*)$ for all $t$, we have that $x \in F'$.
By Lemma~\ref{lemma:window-benchmark-strategies-within-EMD-wT}, this implies that $x \in G(O(wT))$ as well.
Since $\abs{F} = O((\varepsilon/\abs{A})^{-\abs{A}})$, we have by Theorem~\ref{theorem:EMD-main-result} that the regret against $x$ of choosing actions according to $\texttt{LagrangianEMD}(F)$ with $D = O(wT)$ and $\overline{\lambda} = \alpha$ until no budget remains (and then choosing the null action for the remaining time steps) is at most $\alpha \cdot O(\abs{A} \sqrt{T\log(\abs{A}/(\varepsilon\delta))} + \sqrt{wT})$ with probability at least $1 - \delta$.
\end{proof}

\subsection{Best Fixed Distribution in Each Disjoint Interval and Per-Interval Budget Constraints}
\label{sec:benchmark-interval}

In this section, we show that \texttt{LagrangianEMD} achieves the optimal regret guarantee (up to logarithmic terms)  against the following new pacing benchmark.
\begin{align*}
     \textstyle \textsc{Opt} \coloneq \sum_{k=0}^{T/w - 1} \max_{x \in \Delta(A)} \left\{\sum_{t = kw+1}^{(k+1)w} r_t(x) : \sum_{t = kw+1}^{(k+1)w}  c_t(x) \leq Bw/T\right\}
\end{align*}
That is, divide the entire time horizon of length $T$ into $T/w$ disjoint intervals of length $w$.
$\opt$ is the reward obtained by the strategy that in each interval, chooses actions according to the best fixed distribution over actions subject to spending at most a $w/T$ fraction of the budget.

\begin{theorem}
\label{thm:window-EMD-algo}
Let $\Delta(A)_\varepsilon$ denote an $\varepsilon$-cover (with respect to the $\ell_\infty$ norm) of $\Delta(A)$ of minimum size, and let $F \subseteq \Delta(A)_{\varepsilon / \abs{A}}^T$ denote the set of strategies that choose actions according to a fixed distribution in $\Delta(A)_{\varepsilon / \abs{A}}$ in each disjoint interval of length $w$.
If there exists $\alpha \geq 0$ such that $r_t(a) \leq \alpha \cdot c_t(a)$ for all time steps $t$ and actions $a$, then the regret against $\opt$ of choosing actions according to $\texttt{LagrangianEMD}(F)$ with $D = O(wT)$ and $\overline{\lambda} = \alpha$ until no budget remains (and then choosing the null action for the remaining time steps) is at most $\alpha \cdot O(\abs{A} T \sqrt{\log(\abs{A}/\varepsilon) / w} \cdot \log(1/\delta) + \sqrt{wT} + \varepsilon T)$ with probability at least $1 - \delta$ for all $\delta$ simultaneously.
If $\varepsilon = \sqrt{w/T}$, then the regret $\leq \alpha \cdot O(\abs{A} T \sqrt{\log(T \cdot \abs{A} / w) / w} \cdot \log(1/\delta) + \sqrt{wT})$ with probability at least $1 - \delta$ for all $\delta$ simultaneously.
\end{theorem}

\begin{proof}
Let $F' \subseteq \Delta(A)^{T}$ denote the set of strategies that in each disjoint interval of length $w$, choose actions according to a fixed distribution that spends at most $Bw/T$ in the interval.
Let $x^* = \arg\max_{x \in F'} \sum_t r_t(x_t)$ denote the optimal strategy among such strategies, and let $x \in F$ denote a strategy such that $x^*_t(a) - x_t(a) \in [0, 2\varepsilon/\abs{A}]$ for all time steps $t$ and actions $a$ (put all remaining probability mass on the null arm).
Note that $\sum_{t=1}^T r_t(x_t) \geq \sum_{t=1}^T r_t(x^*_t) - 2 \varepsilon T$.
Thus, it suffices to bound our regret against $x$.

To this end, note that $c_t(x_t) \leq c_t(x_t^*)$ for all $t$, so $x \in F'$.
By Lemma~\ref{lemma:window-benchmark-strategies-within-EMD-wT}, this implies that $x \in G(O(wT))$ as well.
Since $\abs{F} = O((\varepsilon/\abs{A})^{-\abs{A} \cdot T / w})$, we have by Theorem~\ref{theorem:EMD-main-result} that the regret against $x$ of choosing actions according to $\texttt{LagrangianEMD}(F)$ with $D = O(wT)$ and $\overline{\lambda} = \alpha$ until no budget remains (and then choosing the null action for the remaining time steps) is at most $\alpha \cdot O(\abs{A} T \sqrt{\log(\abs{A}/(\varepsilon)) / w} \cdot \log(1/\delta)) + \sqrt{wT})$.
\end{proof}

We now show that the regret guarantee against $\opt$ of $\texttt{LagrangianEMD}$ is optimal with respect to the parameters $w$ and $T$ up to logarithmic terms.

\begin{restatable}{theorem}{windowLowerBound}\label{theorem:lower-bound}
There exists an instance such that for any strategy $x \in \Delta(A)^T$, 
\begin{align*}
    \reg_{\opt}(x) + V_T(x)=\Omega(T/\sqrt{w} + \sqrt{wT})
\end{align*}
In particular, for any budget-feasible strategy $x \in \Delta(A)^T$, that is, $V_T(x) = 0$, 
\[
    \reg_{\opt}(x) =\Omega(T/\sqrt{w} + \sqrt{wT})
\]    
\end{restatable}

\begin{proof}
{\em [Sketch; a complete proof can be found in Appendix~\ref{App:thm5}]}~~The lower bound consists of two terms corresponding to two regimes. 

In the {\em small window} regime $w<\sqrt{T}$,  the dominant term is $T/\sqrt{w}$. Since the benchmark strategy is allowed to change at every $w$-window, the ``usual" $\Omega(\sqrt{w})$ regret per window applies. Specifically, for each window, let there be two strategies, each costing exactly $B/T$ per round (so there are no budget issues at all in this regime), and whose payoff is a random $B_{1/2}$ at each step. 
Since the payoffs are all i.i.d. random, any online strategy will have an expected payoff of $T/2$. 

On the other hand, choosing the better of the two strategies at each $w$ window gives an expected advantage of $\Omega(\sqrt{w})$ over $w/2$. Thus
$$
 \reg_{\opt}(x) = (T/w)\cdot\left(w/2 + \Omega(\sqrt{w})\right) - T/2 = \Omega(T/\sqrt{w}). 
$$

In the {\em large window} regime $w\ge \sqrt{T}$, the dominant term is $\sqrt{wT}$. The benchmark strategies are allowed to distribute their spending arbitrarily within (fairly large) windows of size $w$, and the regret comes from a  save-or-spend dilemma, repeated in each $w$-window. 

As in other save-or-spend scenarios, we will have two possible non-null action costing $2B/T$ per round (so the player can afford to play exactly half the time). Within each window indexed by $i\in\{1,\ldots,T/w\}$ of length $w$, the first action yields reward of $a_i\in[0,1]$ per round in the first half of the interval (and zero afterwards), while the second action yields $b_i\in[0,1]$ in the second half (and zero in the first half). A benchmark strategy is allowed to spend all its money on one of the strategies, thus yielding
$(w/2)\cdot \sum_i \max(a_i,b_i)$.

One could hope that an online strategy, which does not know whether $a_i>b_i$ or the other way around would only be able to yield 
$(w/2)\cdot \sum_i (a_i+b_i)/2$. Unfortunately, this is generally not the case, because the online strategy is not required to be close to a sub-pacing strategy. For example, the online strategy may simply choose to play in all rounds where $a_i$ or $b_i$ are greater than $0.5$ until the budget is exhausted. We overcome this problem using martingales. 

We let $a_1=0.5$ the sequence $a_1,b_1,\ldots,a_{T/w},b_{T/w}$ be an unbiased random walk with step size $\varepsilon$ (to be selected later), terminated at $0$ or $1$. Then, at any point, the expected future value of budget is equal to its current value. Therefore (in expectation) any online strategy is indifferent between saving and spending. Therefore, the expected payoff of the best online strategy {\em where the expectation is taken over both the randomness of the random walk and the randomness of the strategy} is \begin{equation}\label{eq:1}0.5\cdot(T/2B)\cdot B = T/4=(w/2)\cdot \EE\left[\sum_i (a_i+b_i)/2\right]\end{equation}
On the other hand, the optimal benchmark strategy can choose between $a_i$ and $b_i$, and thus its payoff is 
\begin{equation}
    \label{eq:2}
    OPT = (w/2)\cdot \EE\left[\sum_i \max(a_i,b_i)\right]
\end{equation}
Subtraction \eqref{eq:1} from \eqref{eq:2} we obtain a lower bound of $\Omega(w)\cdot\EE\left[ \sum_i |a_i-b_i|\right]$ on the expected regret. Note that $\EE[|a_i-b_i|]$ is step size $\varepsilon$ times the probability that the random walk does not terminate before $2i$ steps. We can select $\varepsilon=\Theta(\sqrt{w/T})$, to ensure that the probability of not terminating in $2i$ steps for any $i\le T/w$ is at least a constant. Then the overall expected regret is at least 
$$
\Omega(w)\cdot (T/w) \cdot \Theta(\sqrt{w/T})\cdot \Omega(1) = \Omega(\sqrt{wT}).
$$

\paragraph{Remark} It should be noted that a larger step size $\varepsilon$ would lead to a higher regret bound per window, but the walk would terminate after $\Theta(1/\varepsilon^2)\ll T/w$ windows. We could still get a larger global regret bound if we were allowed to ``reset" the random walk. This suggests that any algorithm attaining the optimal $O(\sqrt{wT})$ regret in this regime must have long-term memory about past payoffs. 
\end{proof}

\bibliographystyle{plainnat}
\bibliography{references}

\newpage

\appendix

\section{Complete Proof of Theorem~\ref{theorem:lower-bound}}
\label{App:thm5}

\windowLowerBound*

\begin{proof}
We construct two lower bound instances, one for when $w \leq \sqrt{T}$, another for when $w \geq \sqrt{T}$.

If $w \leq \sqrt{T}$, then $T / \sqrt{w}$ is the dominating term, so we give an instance such that for any strategy $x \in \Delta(A)^T$, $\reg_{\opt}(x)+V_T(x)=\Omega(T/\sqrt{w})$. 
The instance simply repeats the standard lower bound instance from the online learning literature in each disjoint $w$-sized interval. 
More specifically, the learner faces three actions: the null action $\perp$, the ``heads'' action $H$, and the ``tails'' action $L$.
At each time step $t$, the rewards for the actions are determined by a coin flip.
If the coin lands on heads, then $r_t(H) = 1$, while $r_t(L) = 0$.
Otherwise, the coin lands on tails, and $r_t(H) = 0$, while $r_t(L) = 1$.
The cost of each action is 0.
Since the budget constraint never binds, $\opt$ is simply the total expected reward of the best action in each $w$-sized interval, so 
\[
    \textstyle \opt = \frac{T}{w} \cdot (w/2 + \Omega(\sqrt{w})) = T/2 + \Omega(T/\sqrt{w})
\]
Now, consider any strategy $x \in \Delta(A)^T$.
We know that that for all time steps $t$, 
\[
    \EE[r_t(x_t)] = \EE[r_t(H)] \cdot x_t(H) + \EE[r_t(L)] \cdot x_t(L) = 1/2 \cdot (x_t(H) + x_t(L)) \leq 1/2
\]
Thus, $\sum_{t = 1}^T \EE[r_t(x_t)] \leq T/2$. 
It follows that $\reg_{\opt}(x) = \Omega(T/\sqrt{w})$.

Now, if $w \geq \sqrt{T}$, then $\sqrt{wT}$ is the dominating term.
We construct the following instance to show that $\reg_{\opt}(x)+V_T(x)=\Omega(\sqrt{wT})$ for any strategy $x \in \Delta(A)^T$.
In this instance, the learner has budget $B = T/2$ and faces three actions: the null action $\perp$, the ``spend, then save'' action $a_1$, and the ``save, then spend'' action $a_2$.
For all time steps, the reward/cost of the null action is 0.

The rewards of the two non-null actions are generated according to the following stochastic process.
For simplicity, assume that $w$ divides $T$ and that $w$ is even.
Let $\Tilde{R}_{-1} \coloneq 1/2$.
For all $n = 0, \dots, 2T/w - 1$, let $X_n$ denote the random variable that is $\varepsilon$ (which we set later but for simplicity, assume divides $1/2$) with probability $1/2$ and $-\varepsilon$ with probability $1/2$.
Let $\Tilde{R}_n = \Tilde{R}_{-1} + \sum_{k=0}^n X_k$.
Note that $\Tilde{R}_n$ is simply the position of a one-dimensional random walk starting at $\Tilde{R}_{-1}$ after $n$ steps, so $\EE[\Tilde{R}_n] = \Tilde{R}_{-1} = 1/2$.
Let $N$ denote the first time the random walk hits either 0 or 1.
If the random walk never hits either 0 or 1, then we define $N = 2T/w - 1$.
Let $R_n \coloneq \Tilde{R}_n$ for all $n < N$.
For $n \geq N$, we set $R_n \coloneq \Tilde{R}_N$.
We then set the reward and cost of the ``spend, then save'' action at time $t$ to be $r_t(a_1) \coloneq R_{\ceil{2t/w} - 1} \cdot \1(\text{$\ceil{2t/w}$ is odd})$ and $c_t(a_1) \coloneq \1(\text{$\ceil{2t/w}$ is odd})$, respectively.
Meanwhile, the reward and cost of the ``save, then spend'' action at time $t$ is $r_t(a_2) \coloneq R_{\ceil{2t/w}-1} \cdot \1(\text{$\ceil{2t/w}$ is even})$ and $c_t(a_2) \coloneq \1(\text{$\ceil{2t/w}$ is even})$, respectively.
That is, the instance partitions the $T$ time steps into $2T/w$ disjoint intervals of size $w/2$: $[kw/2 + 1, (k+1)w/2]$ for $k = 0, \dots, 2T/w-1$.
Within each interval $[kw/2 + 1, (k+1)w/2]$, the reward of the ``spend, then save'' action is the same and is either $R_k$ if $k$ is even or 0 if $k$ is odd.
Similarly, the cost of the ``spend, then save'' action is 1 if $k$ is even and 0 if $k$ is odd.
The reward and cost of the ``save, then spend'' action is the opposite: if $k$ is even, then $r_t(a_2) = c_t(a_2) = 0$; if $k$ is odd, then $r_t(a_2) = R_k$ and $c_t(a_2) = 1$.

Now, we compute $\opt$. Since $B = T/2$, $\opt$ can only spend $w/2$ in each $w$-sized interval $[kw + 1, (k+1)w]$ for $k = 0, \dots, T/w - 1$.
It is not hard to see that as a result, $\opt$ chooses the ``spend, then save'' action $a_1$ at every time step in $[kw + 1, (k+1)w]$ if $R_{2k} \geq R_{2k + 1}$ and the ``save, then spend'' action $a_2$ at every time step otherwise.
Thus,
\begin{align*}
    \opt - T/4
        &= \textstyle \opt - \frac{w}{4} \sum_{k=0}^{2T/w-1} \EE[\Tilde{R}_k] \tag{$\EE[\Tilde{R}_k] = 1/2$} \\
        &= \textstyle \frac{w}{2} \cdot \sum_{k=0}^{T/w-1} \EE\SqBr{\max\{R_{2k}, R_{2k+1}\} - (\Tilde{R}_{2k} + \Tilde{R}_{2k+1})/2} \\
        &= \begin{aligned}[t]
            & \textstyle \frac{w}{2} \cdot \sum_{k=0}^{T/w-1} \EE\SqBr{(\max\{\Tilde{R}_{2k}, \Tilde{R}_{2k+1}\} - (\Tilde{R}_{2k} + \Tilde{R}_{2k+1})/2) \cdot \1(N \geq 2k+1)} \\
            & + \textstyle \frac{w}{2} \cdot \sum_{k=0}^{T/w-1} \EE\SqBr{(\Tilde{R}_N - (\Tilde{R}_{2k} + \Tilde{R}_{2k+1})/2) \cdot \1(N \leq 2k)}
        \end{aligned} \\
        &= \textstyle \frac{w}{2} \cdot \sum_{k=0}^{T/w-1} \varepsilon/2 \cdot \PP[N \geq 2k+1]
\end{align*}
It remains to bound $\sum_{k=0}^{T/w-1} \PP[N \geq 2k+1]$ from below.
To do so, first note that by Chernoff,
\[
    \textstyle \PP\SqBr{\Tilde{R}_n \not\in [0, 1]} = \PP_{Z_i \sim \mathrm{Bern}(1/2)}\SqBr{\abs{\sum_{i=1}^n Z_i - n/2} \geq 1/(4\varepsilon)} \leq 2\exp\Paren{-\frac{1}{24 \varepsilon^2 n}}
\]
Thus,
\begin{align*}
    \PP\SqBr{N \leq 1/(48\varepsilon^2)} 
        &\leq \textstyle \sum_{n=1/(2\varepsilon)}^{1/(48\varepsilon^2)} \PP\SqBr{\Tilde{R}_n \not\in [0, 1]} \tag{union bound} \\
        &\leq \textstyle 2 \sum_{n=1/(2\varepsilon)}^{1/(48\varepsilon^2)} \exp\Paren{-\frac{1}{24 \varepsilon^2 n}} \\
        &= \textstyle 2 \sum_{n=2}^{1/(12\varepsilon)} e^{-n} \\
        &\leq 1/2
\end{align*}
It follows that 
\[
    \textstyle \opt - T/4 \geq \frac{\varepsilon w}{4} \sum_{k=0}^{\min\{1/(96 \varepsilon^2), T/w\} - 1} \PP[N \geq 2k+1] \geq \min\CrBr{\frac{w}{768 \varepsilon}, \frac{\varepsilon T}{8}}
\]

Now, consider any sequence $x \coloneq (x_t)_t$ of distributions over actions.
Define
\[
    \textstyle Q_t \coloneq \Paren{B - \sum_{s=1}^t c_s(x_s)} R_{\ceil{2t/w} - 1} + \sum_{s=1}^t r_s(x_s)
\]
We show that $(Q_t)_t$ is a martingale.
Let $\mathcal{F}_{t-1}$ denote the natural filtration for $(Q_s)_{s \in [t-1]}$.
\begin{align*}
    \EE\cSqBr{Q_t}{\mathcal{F}_{t-1}}
        &= \textstyle \EE\cSqBr{\Paren{B - \sum_{s=1}^t c_s(x_s)} R_{\ceil{2t/w} - 1} + \sum_{s=1}^t r_s(x_s)}{\mathcal{F}_{t-1}} \\
        &= \begin{aligned}[t]
            &\textstyle \EE\cSqBr{Q_{t-1}}{\mathcal{F}_{t-1}} + \EE\cSqBr{r_t(x_t) - R_{\ceil{2t/w} - 1} \cdot c_t(x_t)}{\mathcal{F}_{t-1}} \\
            &+ \textstyle \Paren{B - \sum_{s=1}^{t-1} c_s(x_s)} \cdot \EE\cSqBr{(R_{\ceil{2t/w} - 1} - R_{\ceil{2(t-1)/w} - 1})}{\mathcal{F}_{t-1}}
        \end{aligned}
\end{align*}
We show that all terms except $\EE\cSqBr{Q_{t-1}}{\mathcal{F}_{t-1}}$ are in fact 0.
Note that by definition of $r_t$ and $c_t$,
\begin{align*}
    \EE & \cSqBr{r_t(x_t) - R_{\ceil{2t/w} - 1} \cdot c_t(x_t)}{\mathcal{F}_{t-1}} \\
        &= \begin{aligned}[t]
            &\EE\cSqBr{\1(\text{$\ceil{2t/w}$ is odd}) \cdot (r_t(a_1) - R_{\ceil{2t/w} - 1}) \cdot x_t(a_1)}{\mathcal{F}_{t-1}} \\
            &+ \EE\cSqBr{\1(\text{$\ceil{2t/w}$ is even}) \cdot (r_t(a_2) - R_{\ceil{2t/w} - 1}) \cdot x_t(a_2)}{\mathcal{F}_{t-1}}
        \end{aligned} \\
        &= 0
\end{align*}
Similarly,
\[
    \EE\cSqBr{(R_{\ceil{2t/w} - 1} - R_{\ceil{2(t-1)/w} - 1})}{\mathcal{F}_{t-1}} = \EE\cSqBr{R_{\ceil{2t/w} - 1}}{\mathcal{F}_{t-1}} - R_{\ceil{2(t-1)/w} - 1} = 0  
\]
where the second inequality follows from the fact that either (1) $\ceil{2t/w} = \ceil{2(t-1)/w}$, in which case $R_{\ceil{2t/w} - 1} = R_{\ceil{2(t-1)/w} - 1}$, (2) $\ceil{2t/w} > \ceil{2(t-1)/w}$ and the random walk has not hit either 0 or 1 by time $(t-1)$, in which case $R_{\ceil{2t/w} - 1}$ is $R_{\ceil{2(t-1)/w} - 1} \pm \varepsilon$ with equal probability, or (3) $\ceil{2t/w} > \ceil{2(t-1)/w}$ and the random walk has hit either 0 or 1 by time $(t-1)$, in which case $R_{\ceil{2t/w} - 1} = R_{\ceil{2(t-1)/w} - 1} = \Tilde{R}_N$.
It follows that $(Q_t)_t$ is a martingale.

Since $(Q_t)_t$ is a martingale, we have that
\[
    \textstyle \EE\SqBr{\Paren{B - \sum_{t=1}^T c_t(x_t)} R_{2T/w - 1} + \sum_{t=1}^T r_t(x_t)} = \EE[Q_T] = Q_0 = B \cdot \Tilde{R}_{-1} = T / 4
\]
Rearranging and noting that $\EE[R_{2T/w - 1}] \in [0, 1]$ yields
\[
    \textstyle \EE\SqBr{\sum_{t=1}^T r_t(x_t)} \leq T/4 + \Paren{\sum_{t=1}^T c_t(x_t) - B}^+ 
\]
It follows that
\[
    \textstyle \reg_{\opt}(x) + V_T(x) \geq \opt - \EE\SqBr{\sum_{t=1}^T r_t(x_t)} + \Paren{\sum_{t=1}^T c_t(x_t) - B}^+  \geq \min\CrBr{\frac{w}{768 \varepsilon}, \frac{\varepsilon T}{8}}
\]
Choosing $\varepsilon = \sqrt{w/T}$ yields the lower bound.
\end{proof}

\section{Necessity of EMD}

Our results imply that it is possible to get sublinear regret against any (sub-exponential) class of strategies whose spending pattern is within EMD $o(T^2)$ of some sub-pacing spending pattern. On the other hand, in the classic spend-or-save counterexample where sublinear regret is not possible, the EMD distance between the relevant spending patterns is $\Omega(T^2)$ (see Section~\ref{sec:benchmark-dilemma}). 

This still raises the question: are there other spending patterns at EMD $\Omega(T^2)$ from a sub-pacing spending pattern that it is possible to compete with? In other words, is EMD the correct metric for characterizing viable spending patterns to compete with?

In this Appendix, we answer this question positively by showing that given \emph{any two} spending patterns at EMD $\Omega(T^2)$ from one another, it is impossible to design a learning algorithm that can guarantee sublinear regret with respect to both of these spending patterns. In fact, this is even possible in setting with only two actions, and where we only need to compete with a family $F$ containing two strategies. 

Before stating the theorem, we introduce a bit of notation to extend our definition of $\opt_{D, F}$ to general classes of spending patterns (i.e., beyond the set of spending patterns at distance $D$ from uniform). Let $\mathcal{C} \subseteq [0, 1]^{T}$ denote a set of spending patterns (we will generally impose $\sum_{t} c_t \leq B$ for any $(c_t)_t \in \mathcal{C}$). Define 

\[
    \textstyle F(\mathcal{C}) \coloneq \CrBr{(x_t)_t \in \Delta(A)^{T}:  (c_t(x_t))_t \in \mathcal{C}}
\]

\noindent
to be the set of strategies that produce one of the spending patterns in $\mathcal{C}$. Similarly, given any set of strategies $F$, define 

\[
     \textsc{Opt}_{\mathcal{C}, F} \coloneq \max_{x \in F \cap F(\mathcal{C})} \sum_{t=1}^T r_t(x_t)
\]
\noindent
to be the optimal performance of a strategy in $F$ following one of the spending patterns in $\mathcal{C}$, and 
\[
    \textsc{Regret}_{\opt_{\mathcal{C},F}}(x) \coloneq \opt_{\mathcal{C}, F} - \sum_{t=1}^T r_t(x_t)
\]

\noindent
to be the regret of a repeated strategy $x$ against this optimal performance. 

\begin{theorem}
Fix any $0 < \alpha < 1$, set $B = \alpha T$ and consider two spending patterns $(c_1, c_2, \dots, c_T)$ and $(c'_1, c'_2, \dots, c'_T)$ satisfying $\sum_{t} c_t = \sum_{t} c'_t = B$ and $\mathrm{EMD}((c_t)_t, (c'_t)_t) = D$. Let $\mathcal{C} = \{c, c'\}$.

Any learning algorithm must incur

\[
    \textsc{Regret}_{\opt_{\mathcal{C}, F}}(x) = \Omega\left(\frac{D}{T}\right)
\]

\noindent
against some sequence of rewards $r_t(x)$ and $c_t(x)$, even for families $F$ with two strategies ($|F| = 2$) and settings with two actions ($A = 2$). 
\end{theorem}
\begin{proof}
Consider the setting with $A = 2$ actions: one null actions $\perp$, and one ``buy'' action $\textsc{buy}$.  If the agent picks the null action, they receive reward zero and pay cost zero, if they pick the $\textsc{buy}$ action in round $t$ they pay cost $1$ and receive reward $R_t$ (for some sequence $R_t$ chosen by the adversary). Note that for mixed actions (which we denote by specifying the probability $x_t \in [0, 1]$ of playing $\textsc{buy}$) we have that $c_{t}(x_t) = x_{t}$ and $r_{t}(x_t) = R_{t}x_t$. In addition, define the family $F = \{(x_t)_t, (x'_t)_t\}$ to be the family containing the following strategies: the first strategy will play $x_t = c_t$ for each $t \in [T]$, the second strategy will play $x'_t = c'_t$. Note that for the setting we are considering $x_t$ will always fit the spending pattern $c_t$ and $x'_t$ will always fit the spending pattern $c'_t$.

Because $\mathrm{EMD}((c_t)_t, (c'_t)_t) = D$, there must be a $\tau \in [T]$ such that $\left|\sum_{t=1}^{\tau} c_t - \sum_{t=1}^{\tau}c'_t \right| \geq D/T$. Without loss of generality, assume  $\sum_{t=1}^{\tau} c_t \geq \frac{D}{T} + \sum_{t=1}^{\tau}c'_t$. Consider the following two adversarial sequences $R_t$:

\begin{itemize}
    \item The first sequence sets $R_{t} = 1/2$ for all $t \leq \tau$ and $R_{t} = 0$ for $t > \tau$.
    \item The second sequence sets $R_{t} = 1/2$ for all $t \leq \tau$ and $R_{t} = 1$ for $t > \tau$.
\end{itemize}

Consider the outcome of running any given learning algorithm against the above two adversarial sequences. Since both sequences are the same for the first $\tau$ rounds, the algorithm must spend the same total amount and receive the same total reward in expectation over both sequences. Assume that the algorithm spends $B_{\tau}$ in expectation over the first $\tau$ rounds (and therefore receives $B_{\tau}/2$ total utility). Against the first adversarial sequence $R_t$, this algorithm therefore receives $B_{\tau}/2$ utility over all rounds, and against the second adversarial sequence, this algorithm receives utility at most $B_{\tau}/2 + (B - B_{\tau}) = B - B_{\tau}/2$ (assuming it does not go over budget).

Let $C_{\tau} = \sum_{t=1}^{\tau} c_t$ and $C'_{\tau} = \sum_{t=1}^{\tau} c'_t$ (so by assumption, $C_{\tau} \geq C'_{\tau} + (D/T)$. Similarly to above, note that the first strategy $(x_t)_t$ in $F$ receives $C_{\tau}/2$ against the first adversarial sequence and $B - C_{\tau}/2$ against the second adversarial sequence. Likewise, note that the second strategy $(x'_{t})_t$ in $F$ receives $C_{\tau}'/2$ against the first adversarial sequence and $B - C'_{\tau}/2$ against the second adversarial sequence.

If the learning algorithm we are considering incurs worst-case regret $\textsc{Regret}_{\opt_{\mathcal{C}, F}}(x) \leq R$, it must therefore be the case that

$$B_{\tau}/2 \geq \max\left(\frac{C_{\tau}}{2}, \frac{C'_{\tau}}{2}\right) - R = \frac{C_{\tau}}{2} - R,$$

\noindent
and

$$B - B_{\tau}/2 \geq \max\left(B - \frac{C_{\tau}}{2}, B - \frac{C'_{\tau}}{2}\right) - R = B - \frac{C'_{\tau}}{2} - R.
$$

\noindent
Summing these two inequalities and rearranging, we obtain

$$2R \geq \frac{C_{\tau} - C'_{\tau}}{2} \geq \frac{D}{2T}.$$

In particular, it follows that any worst-case regret bound $R$ must satisfy $R \geq \Omega(D/T)$, as desired.
\end{proof}

\begin{algorithm}
    \caption{\texttt{LagrangianDIw}} \label{algorithm:solution-disjoint-interval}
    \DontPrintSemicolon
    \KwIn{budget $B$, time horizon $T$, interval length $w$, bound on dual action $\overline{\lambda}$}

    Let $\alg_P$ be EXP3-IX (\cite{neu15-exp4-ix}; tuned to minimize regret against the best fixed action in hindsight over a time horizon of length $w$)\;
    Let $\alg_D$ be OGD over $[0, \overline{\lambda}]$ with $\lambda_0 = 0$ and step size  $\eta = \overline{\lambda}/ \sqrt{wT}$\;
    \For{$k = 0, \dots, T/w - 1$}{
        (Re)initialize $\alg_P$\;
        \For{$t = kw+1, \dots, (k+1)w$}{
            $x_t \gets \alg_P$, $\lambda_t \gets \alg_D$\;
            Draw $a_t \sim x_t$\;
            Incur reward $r_t(a_t)$ and cost $c_t(a_t)$\;
            Feed reward $r_t(a_t) + \lambda_t (B/T - c_t(a_t))$ to $\alg_P$\;
            Feed loss $\lambda_t(B/T - c_t(a_t))$ to $\alg_D$\;
        }
    }
\end{algorithm}

\section{Algorithm \texttt{LagrangianDIw}}

We give an algorithm \texttt{\textbf{LagrangianD}isjoint\textbf{I}ntervalsOfLength\textbf{w}} (\texttt{LagrangianDIw}) that efficiently achieves a regret guarantee against
\begin{align*}
     \textstyle \textsc{Opt} \coloneq \sum_{k=0}^{T/w - 1} \max_{x \in \Delta(A)} \left\{\sum_{t = kw+1}^{(k+1)w} r_t(x) : \sum_{t = kw+1}^{(k+1)w}  c_t(x) \leq Bw/T\right\}
\end{align*}
(unlike \texttt{LagrangianEMD} for this benchmark).
The pseudocode of \texttt{LagrangianDIw} is given in Algorithm~\ref{algorithm:solution-disjoint-interval}.

\begin{theorem}\label{theorem:upper-bound}
If there exists a scalar $\alpha \geq 0$ such that $r_t(a) \leq \alpha \cdot c_t(a)$ for all time steps $t$ and actions $a$, then the regret against $\opt$ of choosing actions according to $\texttt{LagrangianDIw}(D, F)$ with $\overline{\lambda} = \alpha$ until no budget remains (and then choosing the null action for the remaining time steps) is at most $\alpha \cdot O(T\sqrt{\abs{A} \log(\abs{A}) / w} \cdot \log(T/(w\delta)) + \sqrt{wT})$ with probability $1 - \delta$ for all $\delta$ simultaneously.
\end{theorem}

\begin{proof}
Let $a \coloneq (a_t)_t$ denote the actions output by \texttt{LagrangianDIw}.
The regret guarantee of $\alg_P$ implies that with probability at least $1 - w \delta / T$ in each period $k = 0, \dots, T/w - 1$,
\begin{equation}\label{equation:primal-regret}
    \textstyle \sum\limits_{t = kw+1}^{(k+1)w}  ((r_t(a) + \lambda_t(B/T - c_t(a))) - (r_t(a_t) + \lambda_t(B/T - c_t(a_t)))) \leq \overline{\lambda} \cdot O(\sqrt{w \abs{A} \log \abs{A}} \cdot \log(\frac{T}{w\delta}))
\end{equation}
for all actions $a \in A$.
Meanwhile, recall that $\alg_D$ is OGD over $[0, \overline{\lambda}]$ with step size $\eta$.
OGD with these parameters guarantees that
\begin{equation}\label{equation:dual-regret}
    \textstyle \sum_t \lambda_t(B/T - c_t(a_t)) - \overline{\lambda} \left(B-\sum_t c_t(a_t)\right) \leq \frac{\overline{\lambda}^2}{\eta} + \eta T
\end{equation}
Let $S$ denote the last time step before the expenditure of \texttt{LagrangianDIw} exceeds the budget $B$.
That is, $\sum_{t=1}^{S} c_t(a_t) \leq B$ but $\sum_{t=1}^{S+1} c_t(a_t) > B$.
Note that our strategy is $(a_1, \dots, a_S, \perp, \dots, \perp)$, which we denote as $(a_{1 : S}, \perp)$.
Thus, with probability at least $1 - \delta$, our regret is
\begin{align*}
    & \textstyle \reg((a_{1:S}, \perp)) \\
        &= \textstyle \opt - \sum\limits_{t=1}^S r_t(a_t) \\
        &= \textstyle \sum\limits_{k=0}^{T/w - 1} \max\limits_{x \in \Delta(A)} \left\{\sum\limits_{t = kw+1}^{(k+1)w} r_t(x) : \sum\limits_{t = kw+1}^{(k+1)w}  c_t(x) \leq Bw/T\right\} - \sum\limits_{t=1}^S r_t(a_t) \\
        &= \textstyle \sum\limits_{k=0}^{T/w-1} \min\limits_{\lambda \geq 0} \max\limits_{x \in \Delta(A)} \left\{\sum\limits_{t = kw+1}^{(k+1)w}  r_t(x) + \lambda\left(\frac{Bw}{T} - \sum\limits_{t = kw+1}^{(k+1)w}  c_t(x)\right) \right\} - \sum\limits_{t=1}^S r_t(a_t) \tag{Lagrangian duality} \\
        &\leq \textstyle \sum\limits_{k=0}^{T/w-1} \max\limits_{x \in \Delta(A)}  \left\{\sum\limits_{t = kw+1}^{(k+1)w}  r_t(x) + \left(\frac{1}{w}  \sum\limits_{t = kw+1}^{(k+1)w}  \lambda_t \right)  \left(\frac{Bw}{T} - \sum\limits_{t = kw+1}^{(k+1)w}   c_t(x)\right) \right\} - \sum\limits_{t=1}^S r_t(a_t) \\
        &= \textstyle \sum\limits_{k=0}^{T/w-1} \max\limits_{a \in A}  \left\{\sum\limits_{t = kw+1}^{(k+1)w}  r_t(a) + \left(\frac{1}{w}  \sum\limits_{t = kw+1}^{(k+1)w}  \lambda_t \right)  \left(\frac{Bw}{T} - \sum\limits_{t = kw+1}^{(k+1)w}   c_t(a)\right) \right\} - \sum\limits_{t=1}^S r_t(a_t) \\
        &\leq \begin{aligned}[t]
            & \textstyle \sum\limits_{k=0}^{T/w-1} \max\limits_{a \in A}  \left\{\sum\limits_{t = kw+1}^{(k+1)w}  r_t(a) + \lambda_t \left(\frac{B}{T} - c_t(a)\right) \right\} - \sum\limits_{t=1}^T (r_t(a_t) + \lambda_t (B/T - c_t(a_t)) \\
            & \textstyle + \sum\limits_{t=1}^T \lambda_t(B/T - c_t(a_t)) + \sum\limits_{t=S+1}^T r_t(a_t) + \frac{1}{w} \sum\limits_{k=0}^{T/w-1} \sum\limits_{t=kw+1}^{(k+1)w} \sum\limits_{t'=kw+1}^{(k+1)w} \abs{\lambda_{t'} - \lambda_t} 
        \end{aligned} \tag{$B/T - c_t(a) \in [-1, 1]$} \\
        &\leq \textstyle \frac{T}{w} \cdot O(\overline{\lambda}\sqrt{w \abs{A} \log \abs{A}} \cdot \log(\frac{T}{w\delta})) + \sum\limits_{t=1}^T \lambda_t(B/T - c_t(a_t)) + \sum\limits_{t=S+1}^T r_t(a_t) + \eta w T \tag{Equation~\ref{equation:primal-regret}; $\abs{\lambda_t - \lambda_{t+1}} \leq \eta$} \\
        &\leq \textstyle \frac{T}{w} \cdot O(\overline{\lambda}\sqrt{w \abs{A} \log \abs{A}} \cdot \log(\frac{T}{w\delta})) + \sum\limits_{t=1}^T \lambda_t(B/T - c_t(a_t)) + \alpha \sum\limits_{t=S+1}^T c_t(a_t) + \eta w T \tag{$r_t \leq \alpha \cdot c_t$; $\overline{\lambda} = \alpha$} \\
        &\leq \textstyle \frac{T}{w} \cdot O(\overline{\lambda}\sqrt{w \abs{A} \log \abs{A}} \cdot \log(\frac{T}{w\delta})) + \sum\limits_{t=1}^T \lambda_t(B/T - c_t(a_t)) + \alpha \sum\limits_{t=S+1}^T c_t(a_t) + \eta w T \tag{$r_t \leq \alpha \cdot c_t$; $\overline{\lambda} = \alpha$} \\
        &\leq \textstyle \frac{T}{w} \cdot O(\overline{\lambda}\sqrt{w \abs{A} \log \abs{A}} \cdot \log(\frac{T}{w\delta})) + \sum\limits_{t=1}^T \lambda_t(B/T - c_t(a_t)) - \alpha \Paren{B - \sum\limits_{t=1}^Tc_t(a_t)} + \eta w T \tag{definition of $S$} \\
        &\leq \textstyle \frac{T}{w} \cdot O(\overline{\lambda}\sqrt{w \abs{A} \log \abs{A}} \cdot \log(\frac{T}{w\delta})) + \frac{\overline{\lambda}^2}{\eta} + \eta T (w+1) \tag{Equation~\ref{equation:dual-regret}} \\
        &\leq \textstyle \overline{\lambda} \cdot O(T\sqrt{\abs{A} \log (\abs{A}) / w} \cdot \log(\frac{T}{w\delta}) + \sqrt{wT}) \tag{$\eta = \overline{\lambda}/ \sqrt{wT}$}
\end{align*}
\end{proof}

\end{document}